\documentclass{article}

\usepackage{xcolor}
\usepackage{microtype}
\usepackage{graphicx}
\usepackage{subfigure}
\usepackage{booktabs} 

\usepackage{hyperref}

\newcommand{\err}[1]{\tiny{$\pm#1$}}
\newcommand{\best}[1]{\textbf{#1}}

\usepackage{amsmath,amsthm,amssymb,amsfonts}
\usepackage{bm}
\usepackage{mathtools}
\usepackage{enumitem}
\usepackage{subfigure}
\usepackage{todonotes}
\newtheorem{theorem}{Theorem}
\newtheorem{lemma}[theorem]{Lemma}
\newtheorem{definition}{Definition}

\newtheorem{corollary}{Corollary}

\def\eqref#1{equation~\ref{#1}}

\def\1{\bm{1}}

\def\vb{{\bm{b}}}

\def\vh{{\bm{h}}}

\def\vx{{\bm{x}}}

\def\vz{{\bm{z}}}

\def\mI{{\bm{I}}}

\def\mW{{\bm{W}}}

\DeclareMathAlphabet{\mathsfit}{\encodingdefault}{\sfdefault}{m}{sl}
\SetMathAlphabet{\mathsfit}{bold}{\encodingdefault}{\sfdefault}{bx}{n}

\def\gF{{\mathcal{F}}}

\def\gL{{\mathcal{L}}}

\def\gN{{\mathcal{N}}}
\def\gO{{\mathcal{O}}}

\def\gQ{{\mathcal{Q}}}

\def\gT{{\mathcal{T}}}

\def\sP{{\mathbb{P}}}

\def\sR{{\mathbb{R}}}

\newcommand{\E}{\mathbb{E}}

\newcommand{\KL}{D_{\mathrm{KL}}}
\newcommand{\Var}{\mathrm{Var}}

\newcommand{\Cov}{\mathrm{Cov}}

\def\E{\mathbb{E}} 
\def\Esubarg#1#2{\E_{#1}\left[{#2}\right]}

\newenvironment{sproof}{%
  \proof}{\endproof}
\usepackage{thm-restate}

\usepackage{cleveref}
\usepackage{tikz}
\usepackage{dashrule}

\usepackage[accepted]{icml2019}

\icmltitlerunning{Hierarchical Importance Weighted Autoencoders}

\begin{document}

\twocolumn[
\icmltitle{Hierarchical Importance Weighted Autoencoders}




\begin{icmlauthorlist}
\icmlauthor{Chin-Wei Huang}{udem,elem}
\icmlauthor{Kris Sankaran}{udem}
\icmlauthor{Eeshan Dhekane}{udem}
\icmlauthor{Alexandre Lacoste}{elem}
\icmlauthor{Aaron Courville}{udem,cifar}
\end{icmlauthorlist}

\icmlaffiliation{udem}{Mila, University of Montreal}
\icmlaffiliation{elem}{Element AI}
\icmlaffiliation{cifar}{CIFAR member}

\icmlcorrespondingauthor{Chin-Wei Huang}{chin-wei.huang@umontreal.ca}

\icmlkeywords{Machine Learning, ICML}

\vskip 0.3in
]



\printAffiliationsAndNotice{}  

\begin{abstract}
Importance weighted variational inference \citep{burda2015importance} uses multiple i.i.d. samples to have a tighter variational lower bound. 
We believe a joint proposal has the potential of reducing the number of redundant samples,
and introduce a hierarchical structure to induce correlation.
The hope is that the proposals would coordinate to make up for the error made by one another to reduce the variance 
of the importance estimator.
Theoretically, we analyze the condition under which convergence of the estimator variance can be connected to convergence of the lower bound.
Empirically, we confirm that maximization of the lower bound does implicitly minimize variance.
Further analysis shows that this is a result of negative correlation induced by the proposed hierarchical meta sampling scheme, and performance of inference also improves when the number of samples increases.
\end{abstract}

\section{Introduction}

Recent advance in variational inference~\citep{kingma2013auto,rezende2014stochastic} makes it efficient to model complex distribution using latent variable model with an intractable marginal likelihood $p(\vx)=\int_\vz p(\vx|\vz)p(\vz)d\vz$, where $\vz$ is an unobserved vector distributed by a prior distribution, e.g. $p(\vz)=\gN(\mathbf{0},\mI)$. 
The use of an inference network, or encoder, allows for amortization of inference by directly conditioning on the data $q(\vz|\vx)$, to approximate the true posterior $p(\vz|\vx)$. 
This is known as the \textbf{Variational Autoencoder} (VAE).

Normally, learning is achieved by maximizing a lower bound on the marginal likelihood, since the latter is not tractable in general. 
Naturally, one would be interested in reducing the gap between the bound and the desired objective.
\citet{burda2015importance} devises a new family of lower bounds with progressively smaller gap using multiple i.i.d. samples from the variational distribution $q(\vz|\vx)$, 
which they call the \textbf{Importance Weighted Autoencoder} (IWAE). 
\citet{cremer2017reinterpreting,bachman2015training} notice that IWAE can be interpretted as using a corrected variational distribution in the normal variational lower bound. 
The proposal is corrected towards the true posterior by the importance weighting, and approaches the latter with an increasing number of samples. 

Intuitively, when only one sample is drawn to estimate the variational lower bound, the loss function highly penalizes the drawn sample, and thus the encoder. 
The decoder will be adjusted accordingly to maximize the likelihood in a biased manner, as it treats the sample as the real, observed data. 
In the IWAE setup, the inference model is allowed to make mistakes, as a sample corresponding to a high loss is penalized less owing to the importance weight. 

Drawing multiple samples also allows us to represent the distribution at a higher resolution. 
This motivates us to construct a joint sampler such that the empirical distribution drawn from the joint sampler can better represent the posterior distribution. 
More specifically, we consider a hierarchical sampler, whose latent variable $\vz_0$ acts at a meta-level to decide the salient points of the true posterior, which we call the \textbf{Hierarchical Importance Weighted Autoencoder} (H-IWAE).  
Doing so allows for (1) summarizing the distribution using the latent variable~\citep{edwards2016towards}, and (2) counteracting bias induced by each proposal. 
To analyze the latter effect, we look at the variance of the Monte Carlo estimate of the lower bound, and show that maximizing the lower bound implicitly reduces the variance. 
Our main contributions are as follows:
\begin{itemize}
    \item We propose a hierarchical model to induce dependency among the samples for approximate inference, and derive a hierarchical importance weighted lower bound.
    \item We analyze the convergence of the variance of the Monte Carlo estimate of the lower bound, and draw a connection to the convergence of the bound itself. 
    \item Empirically, we explore different weighting heuristics, validate the hypothesis that the joint samples tend to be negatively correlated, and perform experiments on a suite of standard density estimation tasks. 
\end{itemize}

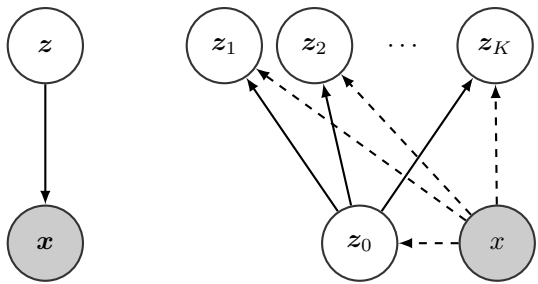
\begin{figure}[!t]
\centering
\subfigure{
\begin{tikzpicture}
\tikzstyle{main}=[circle, minimum size = 10mm, thick, draw =black!80, node distance = 16mm]
\tikzstyle{connect}=[-latex, thick]
  \node[main, fill = white!100] (z) [] {$\vz$};
  \node[main, fill = black!20] (x) [below=of z] {$\vx$};
  \path (z) edge [connect] (x);
\end{tikzpicture}
}
\hspace*{10mm}
\subfigure{
\begin{tikzpicture}
\tikzstyle{main}=[circle, minimum size = 10mm, thick, draw =black!80, node distance = 16mm]
\tikzstyle{connect}=[-latex, thick]
  \node[main, fill = white!100] (z0) [] {$\vz_0$};
  \node[main, fill = white!100] (z1) [above=of z0,xshift=-18mm] {$\vz_1$};
  \node[main, fill = white!100] (z2) [above=of z0,xshift=-6mm] {$\vz_2$};
  \node[main, fill = white!100, draw=black!0] (z3) [above=of z0,xshift=6mm] {$\cdots$};
  \node[main, fill = white!100] (zk) [above=of z0,xshift=18mm] {$\vz_K$};
  \node[main, fill = black!20] (x) [right=of z0,xshift=-8mm] {$x$};
  \path (z0) edge [connect] (z1)
  (z0) edge [connect] (z2)
  (z0) edge [connect] (zk)
  (x) edge [connect, dashed] (z0)
  (x) edge [connect, dashed] (z1)
  (x) edge [connect, dashed] (z2)
  (x) edge [connect, dashed] (zk)
  ;
\end{tikzpicture}
}
\caption{Latent variable model (left) with hierarchical proposals as the inference model (right). Dashed lines indicate amortization.}
\label{fig:hiwae}
\end{figure}
\begin{figure*}[!th]
\centering
\subfigure{
\includegraphics[width=0.9\textwidth]{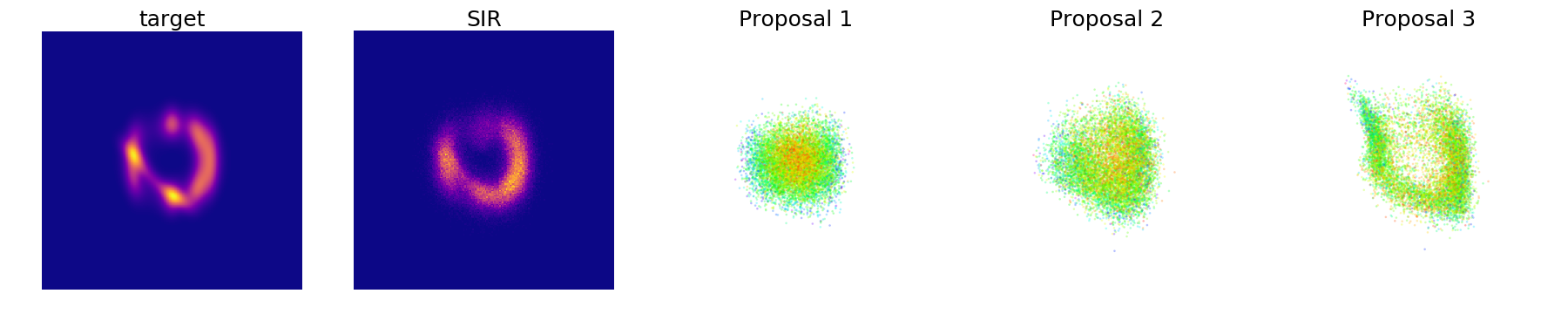}
\label{fig:hiwae_mk}
}\\
\hdashrule[0.5ex][c]{0.8\textwidth}{1.0pt}{1.5mm}
\subfigure{
\includegraphics[width=0.9\textwidth]{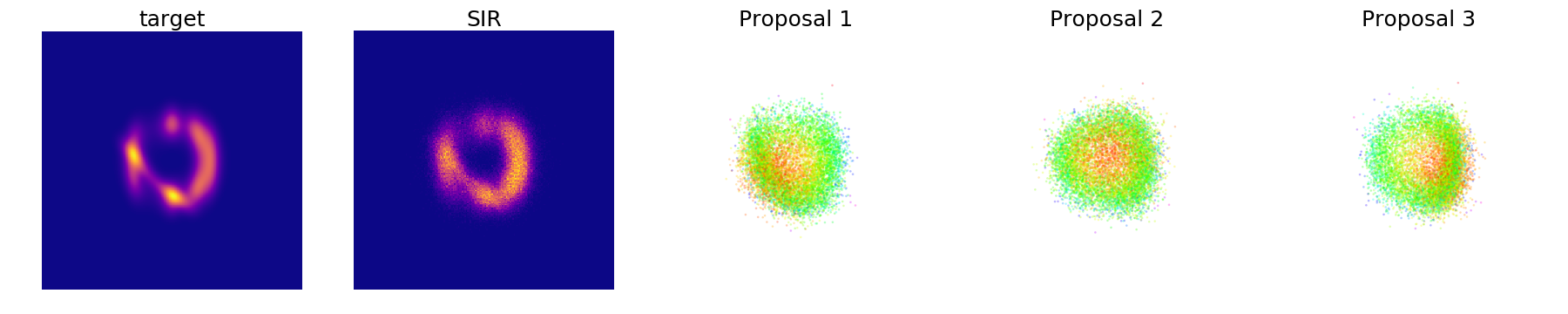}
\label{fig:hiwae_ar}
}\\
\subfigure{
\includegraphics[width=0.9\textwidth]{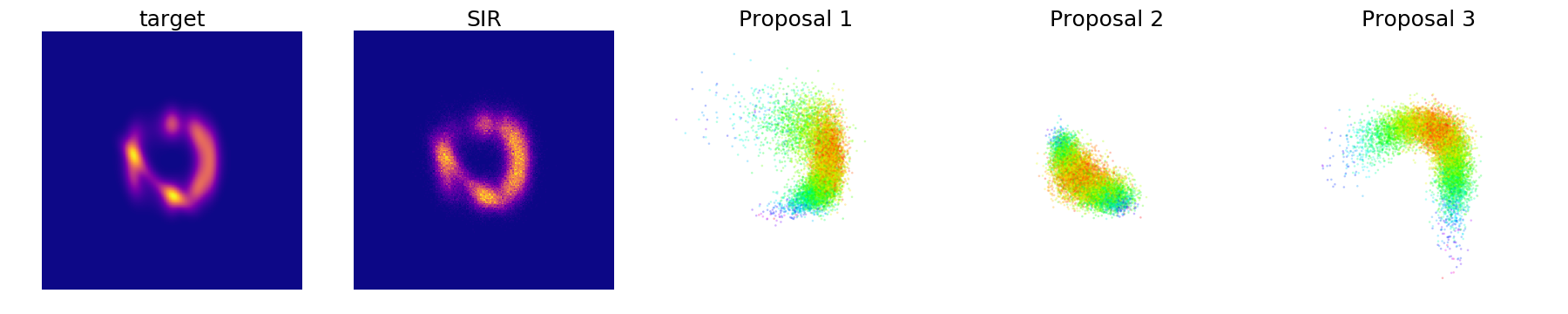}
\label{fig:hiwae_bh}
}

\caption{Learned hierarchical importance sampling proposals with uniform weighting (2nd row) and balanced heuristic (3rd row).
From left to right are the target distribution, heat map of samples drawn from a \textit{sampling importance resampling} procedure, and the proposals. 
The color scheme corresponds to the norm of $\vz_0$, indicating the correlation among the samples. 
A proposal with Markov transition is presented in the 1st row as a comparison: the dependency among the samples is visually much weaker.}
\label{fig:hiwae_toy}
\end{figure*}

\section{Background}
In a typical setup of latent variable model, we assume a joint density function that factorizes as  $p(\vx,\vz)=p(\vx|\vz)p(\vz)$, where $\vx,\vz$ are the observed and unobserved random variables, respectively. 
Learning is achieved by maximizing the marginal log-likelihood of the data $\log p(\vx)=\log\int_\vz p(\vx,\vz) d\vz$, which is in general intractable due to the integration, since a neural network is usually used to parameterize the likelihood function $p(\vx|\vz)$.

\subsection{Variational Inference}
\label{sec:vi}
One way to estimate the marginal likelihood is via the variational lower bound. 
Concretely, we introduce a variational distribution $q(\vz)$ \footnote{For notational convenience we omit the conditioning on $\vx$ for amortized inference.}, and maximize the bound on the RHS:
\begin{align*}
\log p(\vx) 
&=\log \E_{q(\vz)}\left[\frac{p(\vx,\vz)}{q(\vz)}\right] 
\\
&\geq \E_{q(\vz)}\left[\log \frac{p(\vx,\vz)}{q(\vz)}\right] := \gL(q)
\end{align*}
which is known as the evidence lower bound (ELBO). 
The tightness of the ELBO can be described by the reverse Kullback-Leibler (KL) divergence $
\KL(q(\vz)||p(\vz|\vx))$, which is equal to $0$ if and only if $q(\vz)=p(\vz|\vx)$. 

\subsection{Importance Weighted Auto-Encoder}
\label{sec:iwae}
\citet{burda2015importance} noticed the likelihood ratio in the ELBO, $p(\vx,\vz)/q(\vz)$, resembles the importance weight in importance sampling, and introduced a family of lower bounds called the Importance Weighted Lower Bound (IWLB):
\begin{align*}
\label{eq:2:2:1:iwaeExpre}
\log p(\vx) 
&=\log \E_{\vz_j\sim q(\vz_j)}\left[\frac{1}{K}\sum_{j=1}^K\frac{p(\vx,\vz_j)}{q(\vz_j)}\right] 
\\
&\geq \E_{\vz_j\sim q(\vz_j)}\left[\log \frac{1}{K}\sum_{j=1}^K\frac{p(\vx,\vz_j)}{q(\vz_j)}\right] := \gL_K(q)
\end{align*}
with $\gL_K(q)=\gL(q)$ when $K=1$.
In practice, the expectation over the product measure is estimated by sampling $\vz_j\sim q(\vz)$ for $j=1,...,K$, and setting 
$$\tilde{\gL}_K(q)=\log \frac{1}{K}\sum_{j=1}^K\frac{p(\vx,\vz_j)}{q(\vz_j)}$$
One property of this family of bounds is monotonicity; i.e. $\gL_M(q)\geq \gL_N(q)$ for $M>N$. 
Strong consistency of $\tilde{\gL}_K(q)$ can also be proved since $\vz_j$'s are independently and identically distributed (by law of $q(\vz)$), in which case $\tilde{\gL}_K(q)\overset{\scriptscriptstyle{K\rightarrow\infty}}{\longrightarrow} \log p(\vx)$ almost surely.

This asymptotic property motivates the use of large number of samples in practice, but there are two practical concerns:
First, even though evaluation of $\vz_j$ under the generative model $p(\vx,\vz)$ can be done in parallel (sublinear rate in general), memory scales in $\gO(K)$.
An online algorithm can be adopted to reduce memory to $\gO(1)$, but at the cost of $\gO(n)$ evaluation \citep{huangsequentialized}. 
Second, for any finite $K$, for any choice of proposal such that $q(\vz)$ is not proportional to $p(\vx,\vz)$ \footnote{which is a reasonable assumption given the finite approximating capacity of the chosen family of $q$ and the finiteness of the recognition model for amortization.}, $\gL_K(q)$ is strictly a lower bound on the marginal likelihood.
This motivates the research in improving the surrogate loss functions (such as $\gL_K(q)$) for the intractable $\log p(\vx)$.


\citet{nowozindebiasing}, for instance, interpreted $\tilde{\gL}_K(q)$ as a biased estimator of $\log p(\vx)$, and introduced a family of estimators with reduced bias. 
The bias of the new estimator can be shown to be reduced to $\gO(K^{-(m+1)})$, for $m<K$ (compared to $\gO(K^{-1})$ for $\tilde{\gL}_K(q)$). 
However, the variance of the estimator can be high and is no longer a lower bound. 

Alternatively, one can look at the dispersion \footnote{By dispersion, we mean how ``spread out" or ``stretched" the distribution of the random variable is. Common statistical dispersion indices include variance, mean absolute difference, entropy, etc.} of the likelihood ratio $w_j=\frac{p(\vx,\vz_j)}{q(\vz_j)}$, which itself is a random variable. 
The gap between $\log p(\vx)$ and $\gL_K(q)$ can be explained by Jensen's inequality and the fact that $\log$ is a strictly concave function: $\log \E[w] \geq \E[\log w]$ where $w$ is a positive random variable. 
The equality holds if and only if $w$ is almost surely a constant. 
This can be approximately true when $w$ is taken to be the likelihood ratio $\frac{p(\vx,\vz)}{q(\vz)}$ and when $q(\vz)\approx p(\vz|\vx)$. 
Instead, if we take $w=\frac{1}{K}\sum_{j=1}^Kw_j$, we see a direct reduction in variance if $w_j$'s are all uncorrelated and identically distributed: $\Var(w)=\frac{1}{K}\Var(w_1)$. 
Intuitively, the shape of the distribution of $w$ becomes sharper with larger $K$, resulting in a smaller gap when Jensen's inequality is applied. 
This idea was also noticed by \citep{domke2018importance} and explored by \citep{klys2018joint}. 

However, a clear connection between how well $\log p(x)$ is approximated and the minimization of the variance of $w$ and/or the variance of $\log w$ is lacking. 
We defer the discussion to Section~\ref{sec:varmin} wherein we also provide a theoretical analysis. 

\subsection{Variance Reduction via Negative Correlation}
\label{sec:varred}
Let $w = \frac{1}{K}\sum_{i = 1}^K\pi_iw_i$ where $w_i$'s are random variables and $\pi_i$'s sum to one.
The variance of $w$ can be written as: 
\begin{align*}
    \Var(w) = &\sum_{i = 1}^K\pi_i^2\Var(w_i)
    + 2\sum_{i < j}\pi_i\pi_j\Cov(w_i, w_j)
\end{align*}
This suggests that the variance of $w$ is smaller if $w_i$'s are negatively correlated with each other.  
The intuition of this is as follows: if $w_1$ deviates from its mean from below, the error it makes can be canceled out by $w_2$ if the latter is above its mean. 

For example, \textit{antithetic variate}~\citep{mcbook} is a classic technique that relies on negative correlation.
Assume we want to estimate $\E_q[w(\vz)]$, where $q(\vz)$ is a symmetric density. 
Given $\vz\sim q(\vz)$, we can augment our estimate with an $\vz'$ that is opposite to $\vz$ by reflecting through some center point, and set $\bar{w}=\frac{w(\vz)+w(\vz')}{2}$.
$\bar{w}$ is still an unbiased estimator of $\E_q[w(\vz)]$, but has variance $\Var(\bar{w})=\frac{\sigma^2_w}{2}(1+\rho)$, where $\sigma^2_w=\Var(w(\vz))$ and $\rho$ is the correlation between $w(\vz)$ and $w(\vz')$.  
If $w(\cdot)$ is monotonic, then the correlation is negative, so we achieved variance reduction at a rate faster than averaging two samples. 
Thus, \citet{wu2018differentiable} propose to train an antithetic sampler.
But in general, there's no guarantee that the random function $w(\cdot)$ is monotonic, so we propose to train a joint sampler via latent variable.


\section{Hierarchical Importance Weighted Auto-Encoder}
\label{sec:hiwae}
In order to achieve anticorrelation just described, we consider a joint sampling scheme, 
with a hierarchical structure that admits fast sampling and allows us to approximate marginal densities required to form a valid lower bound. 
\subsection{Joint Importance Sampling}
\label{subsec:jims}
We first consider a joint density of $K$ $\vz_j$'s, denoted by $Q(\vz_1, ..., \vz_K)$. 
Let $q_j(\vz_j)=\int_{\vz_{\neg j}} Q(\vz_1, ..., \vz_K) d \vz_{\neg j}$ be the marginal. 
Let $\pi_j(\vz)$ be a weighting factor such that $\sum_{j=1}^K\pi_j(\vz)=1$ for all $\vz$. 
Then Jensen's inequality gives
\begin{align*}
\log p(\vx) &= \log \int_{\small \vz_1,...,\vz_K} \sum_{j=1}^K\pi_j(\vz_j) \frac{p(\vx,\vz_j)}{q_j(\vz_j)} d Q({\small \vz_1,...,\vz_K}) \\
&\geq \E_{Q} \left[\log\sum_{j=1}^K\pi_j(\vz_j) \frac{p(\vx,\vz_j)}{q_j(\vz_j)}\right]
:= \gL_K(Q)
\end{align*}
which we call the Joint Importance Weighted Lower Bound (J-IWLB) (see Appendix~\ref{app:jiwlb} for a detailed derivation). 
This allows us to generalize $\gL_K(q)$ in two ways:
\begin{enumerate}[label=\large\protect\textcircled{\small\arabic*}]
    \item The flexibility to have different marginals
    \item The dependency among $\vz_j$'s
\end{enumerate}
Point {\large\protect\textcircled{\small 1}} with $K>1$ allows us to relax the necessary condition for optimality when $K=1$, i.e. $q\propto p$. 
For example, let $\pi_j(\vz)$ be the ``posterior probability" of the random index $j$, $\pi_j(\vz)=\frac{q_j(\vz)}{\sum_iq_i(\vz)}$. 
Then $\gL_K(Q)$ is equivalent to using a mixture proposal $\sum_{i=1}^K \frac{q_i}{K}$. 
In order for the proposals to be optimal, it is sufficient if $p$ can be decomposed as a mixture density (to wit, $q_j$'s do not all need to be equal to $p$). 

Point {\large\protect\textcircled{\small 2}} allows us to leverage the correlation among $\vz_j$'s to further reduce the variance. 
With $\vz_1$ making a positive deviation from the mean $\frac{p(\vx,\vz_j)}{q_1(\vz_j)}> p(\vx)$, one would hope $\vz_2$ has a higher chance of making a negative deviation to cancel the error.
This can be thought of as a soft version of antithetic variates.

One difficulty in optimizing $\gL_K(Q)$ lies in estimating the marginal density $q_j$. 
Particular choices of $Q$, such as multivariate normal distribution, allows for exact evaluation, since $q_j$ is still normally distributed; this was explored by \citet{klys2018joint}. 
Another option is to define $K-1$ set of invertible transformations, $\gT_{j}(\cdot)$ for $j=2,...,K$, and then apply the mapping $\vz_j\leftarrow \gT_{j}(\vz_{j-1})$, where $\vz_1\sim q_1(\vz)$. 
The density of $z_j$ under $q_j$ can then be evaluated via change of variable transformation. 
This is known as the \textit{normalizing flows} \citep{rezende2015variational}. 
Other more general forms of the joint $Q$, however, does not have tractable marginal densities, and thus require further approximation.

\subsection{Hierarchical Importance Sampling} 
In order to induce correlation among $\vz_j$'s, we consider a hierarchical proposal:
\begin{align*}
Q(\vz_1,...,\vz_K) &= \int_{\vz_0} Q(\vz_1,...,\vz_K|\vz_0) dq_0(\vz_0) \\
&= \int_{\vz_0} \prod_{j=1}^K q_j(\vz_j|\vz_0) dq_0(\vz_0)
\end{align*}
with the conditional independence assumption $z_i\perp z_j \mid z_0$ for $i\neq j$ and $1\leq i, j \leq K$ (see Figure~\ref{fig:hiwae}). 
First, notice that the marginals are not identical since each $\vz_j$ is sampled from a different conditional $q_{j}(\vz_j|\vz_0)$.
More specifically, each marginal is a latent variable model, which can be thought of as an infinite mixture~\citep{ranganath2016hierarchical}: $q(\vz_j)=\int_{\vz_0} q_j(\vz_j|\vz_0)d q_0(\vz_0)$.
Second, $\vz_1,...,\vz_K$ are entangled through sharing the same \textit{common random number} $\vz_0$. 
The smaller the conditional entropy $H(\vz_j|\vz_0)$ is, the more mutually dependent $\vz_j$'s are. 
We analyze the effect of this common random number in Section~\ref{sec:exp1}.

While there are different ways to model the joint proposal, we emphasize the following properties of the hierarchical proposals that make it an appealing choice:
\begin{enumerate}
    \item Sampling can be parallelized.
    \item Empirically, we found that optimization behaves better than a sequential model (we also tested a Markov joint proposal, i.e. $\vz_j$ depends only on $\vz_{j-1}$, and found the learned proposals do not compensate for each other; see top rule of Figure~\ref{fig:hiwae_toy}).
    \item $\vz_0$ can be interpreted as a summary~\citep{edwards2016towards} of the empirical distribution. 
\end{enumerate}
\paragraph{Optimization Objective}
In the spirit of the variational formalism, we need to define an objective to optimize $Q$.
However, the marginal densities are no longer tractable due to the integration over the prior measure $q_0$. 
We introduce an auxiliary random variable \citep{agakov2004auxiliary} $r(\vz_0|\vz_j)$ to approximate $q_j(\vz_0|\vz_j)$ \footnote{$q_j(\vz_0|\vz_j)$ is the posterior of $\vz_0$ given $\vz_j$; that is $q_j(\vz_0|\vz_j)\propto q_j(\vz_j|\vz_0)q_0(\vz_0)$.}, and maximize the following 
objective:
\begin{align*}
\gL_K(Q^0) := \E_{Q^0} \left[\log \sum_{j=1}^K\pi_j(\vz_j,\vz_0) \frac{p(\vx,\vz_j)r(\vz_0|\vz_j)}{q_j(\vz_j|\vz_0)q_0(\vz_0)} \right] 
\end{align*}
where $Q^0$ is the joint of $\vz_0,...,\vz_K$, and $\pi_j(\vz,\vz_0)$ is a partition of unity for all $\vz,\vz_0$.
First, if we choose $\pi_j=\frac{1}{K}$ and let $r$ depend on $j$, $\gL_K(Q^0)$ boils down to the J-IWLB $\gL_K(Q)$ if $r_j(\vz_0|\vz_j)=q_j(\vz_0|\vz_j)$.
Second, with $K=1$, it is simply variational inference with a hierarchical model ~\citep{ranganath2016hierarchical}.
Lastly, $\gL_K(Q^0)$ is a lower bound on $\log p(\vx)$ (we relegate the proof to Appendix~\ref{app:hiwlb}).
We call it the Hierarchical Importance Weighted Lower Bound (H-IWLB). 

This training criterion is chosen also because the inference model and generative model can have a unified objective. 
It is also possible to consider training the inference model with different objectives.
For instance, direct minimization of the variance of the importance ratio amounts to minimization of the $\chi^2$-divergence (see \citet{dieng2017variational,muller2018neural}) \footnote{We have also tried to minimize the variance of the estimator (average of importance ratio), with an estimate of the first moment. 
But we found even the reparameterized gradient suffers from noisy signal and usually converges to a suboptimal solution.}.
We discuss the connection between vanishing variance and convergence of the $\log$ estimator in Section~\ref{sec:varmin}.

\paragraph{Weighting Heuristics}
A family of weighting heuristics commonly used in practice are the \textit{power heuristics}:
$$\pi_j^\alpha(\vz,\vz_0) = \frac{q_j(\vz|\vz_0)^\alpha}{\sum_{i=1}^Kq_i(\vz|\vz_0)^\alpha}$$
With larger $\alpha$, the heuristic tends to emphasize the proposal under which $\vz$ has larger likelihood more. 
When $\alpha=0$ and $1$, $\pi_j$ boils down to uniform probability (arithmetic average of the likelihood ratio) and the posterior probability of the index $j$, respectively.  
We compare different choices of weighting heuristics qualitatively on a toy example in Section~\ref{sec:exp2} and quantitatively on a suite of standard density modeling tasks with latent variable model in Section~\ref{sec:exp3}.

\paragraph{Training Procedure}
In our experiments, all $q_j$ and $q_0$ are conditional Gaussian distributions (also conditioned on $\vx$ in the amortized inference setting).
This allows us to use the reparameterized gradient~\citep{kingma2013auto,rezende2014stochastic}. 
However, as noted by \citet{rainforth2018tighter}, the signal-to-noise ratio of the gradient estimator for IWAE's encoder decreases as $K$ increases, large $K$ would render the learning of the inference model inefficient. 
Thus, we consider the recently proposed doubly reparameterized gradient by \citet{tucker2018doubly}. 
Empirically, we find the algorithm converges more stably.

\begin{figure*}[h]
\centering
\includegraphics[width=0.95\textwidth]{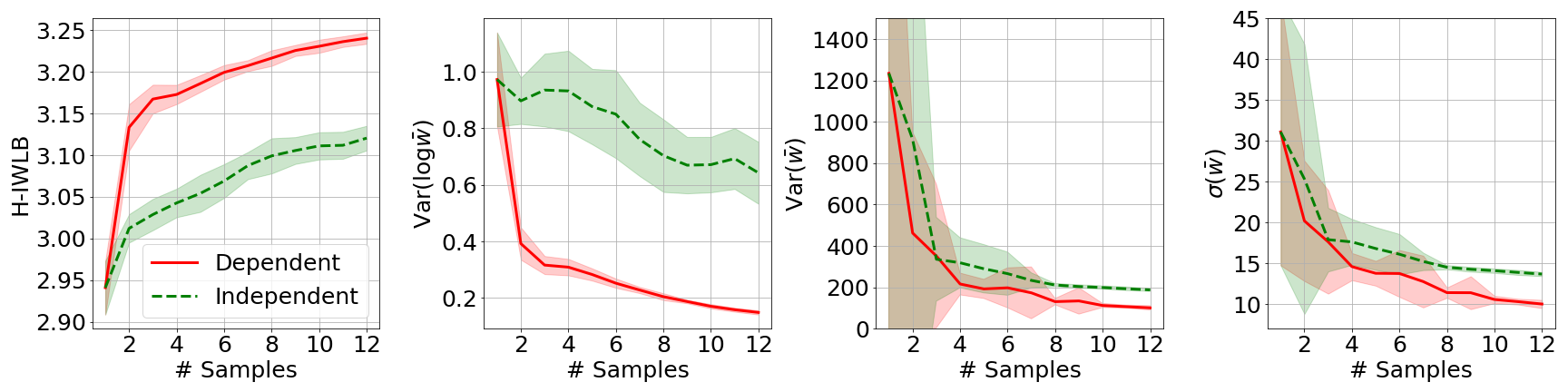}
\caption{\small H-IWLB trained with (dependent) and without (independent) common random number. The hierarchical proposal trained with the common factor learns to avoid redundant sampling and effectively reduces dispersion (lower variance and standard deviation).}
\label{fig:hiwae_dep_toy}
\end{figure*}
\begin{figure*}[t!]
\centering
\subfigure[Uncorrelated $w_j$]{
\centering
\includegraphics[width=0.30\textwidth]{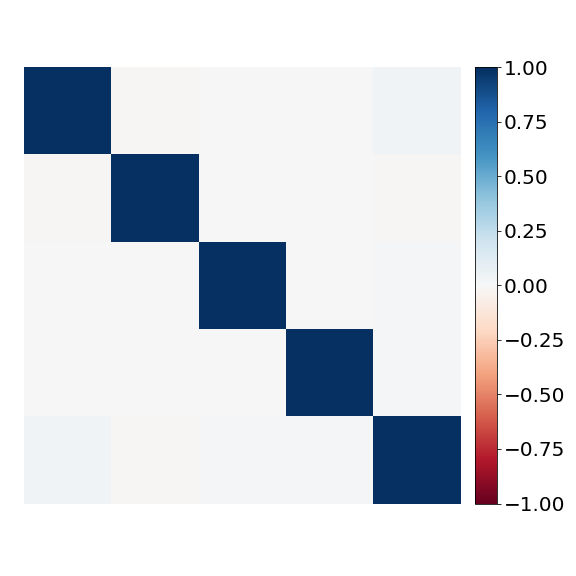}
\label{fig:corr_uncorr_uncorr}}
\hfill
\subfigure[Trained with independent $\vz_0$]{
\centering
\includegraphics[width=0.30\textwidth]{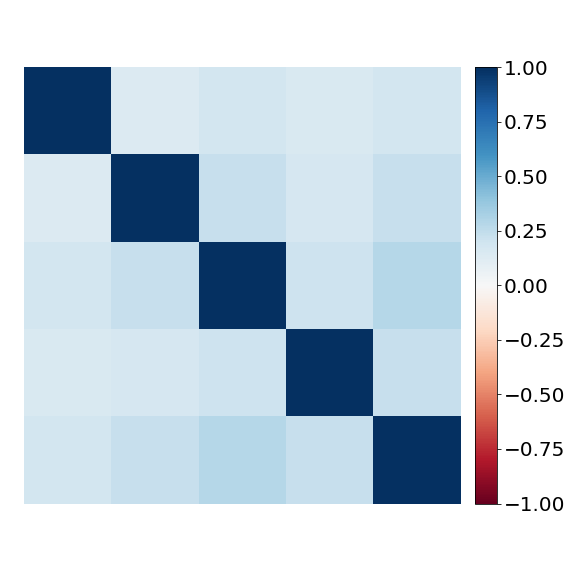}
\label{fig:corr_uncorr_corr}}
\hfill
\subfigure[Trained with common $\vz_0$]{
\centering
\includegraphics[width=0.30\textwidth]{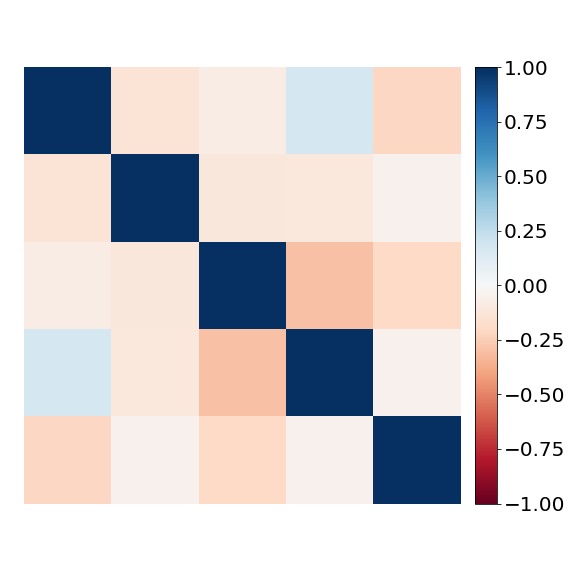}
\label{fig:corr_corr}}
\caption{\small Correlation matrix of $\bar{w}_j$ with independent $\vz_0$ (a) and common $\vz_0$ (b,c). 
When trained with common $\vz_0$ (c), the proposals coordinate to make up for the bias made by one another, resulting in negative correlation.}
\label{fig:hiwae_corr_toy}
\end{figure*}

\begin{table*}[h]
	\centering
	\scriptsize
	\caption{Variational Autoencoders trained on the statically binarized MNIST dataset~\citep{larochelle2011}.
    Each hyperparameter is run 5 times to carry out mean and standard deviation (uncertainties are reported in Appendix~\ref{app:adde} for readability).
    $\tilde{\gL}_{*}$ stands for the lower bound on log likelihood of the dataset ({$tr$}aining, {$va$}lidation and {$te$}st). $\text{NLL}_{*}$ is the negative log likelihood estimated with 2000 samples.}
    \label{tb:mnist}
\begin{tabular}{c|cc|cccc|cccc|cccc}
\toprule
mode               &   IWAE &        H-IWAE &   IWAE &   &        H-IWAE &  &   IWAE &   &        H-IWAE &        &          IWAE &   &        H-IWAE &  \\
$\alpha$           &      - &             0 &      - &       0 &             1 &       3 &      - &       0 &             1 &             3 &             - &       0 &             1 &       3 \\
$K$                &      1 &             1 &      2 &       2 &             2 &       2 &      5 &       5 &             5 &             5 &            10 &      10 &            10 &      10 \\
\hline
$-\tilde{\gL}_{tr}$ &  83.26 &  \best{82.92} &  82.36 &   82.19 &  \best{82.15} &   82.43 &  81.48 &   82.25 &  \best{81.28} &         81.32 &  \best{80.85} &   82.30 &         80.89 &   81.17 \\
$-\tilde{\gL}_{va}$ &  86.57 &  \best{85.75} &  85.40 &   85.05 &  \best{85.03} &   85.18 &  84.45 &   85.05 &  \best{84.04} &         84.12 &         83.84 &   85.11 &  \best{83.77} &   83.95 \\
$-\tilde{\gL}_{te}$ &  86.36 &  \best{85.50} &  85.16 &   84.79 &  \best{84.76} &   84.90 &  84.25 &   84.85 &  \best{83.79} &         83.90 &         83.62 &   84.86 &  \best{83.56} &   83.72 \\
$\text{NLL}_{va}$  &  82.64 &  \best{82.24} &  82.03 &   81.93 &  \best{81.88} &   81.96 &  81.63 &   82.19 &         81.42 &  \best{81.39} &         81.37 &   82.42 &  \best{81.28} &   81.33 \\
$\text{NLL}_{te}$  &  82.37 &  \best{81.96} &  81.77 &   81.65 &  \best{81.60} &   81.66 &  81.37 &   81.93 &         81.16 &  \best{81.13} &         81.13 &   82.17 &  \best{81.04} &   81.08 \\
\bottomrule
\end{tabular}
\end{table*}

\begin{table*}[h]
	\centering
	\scriptsize
	\caption{Variational Autoencoders trained on the dynamically binarized OMNIGLOT dataset~\citep{lake2015human}.
    Each hyperparameter is run 5 times to carry out the mean and standard deviation (uncertainties are reported in Appendix~\ref{app:adde} for readability).}
    \label{tb:omniglot}
\begin{tabular}{c|cc|cccc|cccc|cccc}
\toprule
mode               &    IWAE &         H-IWAE &           IWAE &         &         H-IWAE &         &    IWAE &                &         H-IWAE &         &           IWAE &         &  H-IWAE &         \\
$\alpha$           &       - &              0 &              - &       0 &              1 &       3 &       - &              0 &              1 &       3 &              - &       0 &       1 &       3 \\
$K$                &       1 &              1 &              2 &       2 &              2 &       2 &       5 &              5 &              5 &       5 &             10 &      10 &      10 &      10 \\
\hline 
$\tilde{\gL}_{tr}$ &  106.40 &  \best{104.57} &  \best{102.66} &  103.24 &         102.88 &  103.17 &  102.01 &  \best{101.35} &         101.37 &  102.78 &   \best{99.71} &   99.75 &  100.81 &  101.29 \\
$\tilde{\gL}_{va}$ &  109.05 &  \best{106.48} &         105.85 &  105.27 &  \best{105.18} &  105.56 &  104.48 &         103.76 &  \best{103.72} &  104.70 &  \best{102.67} &  103.21 &  103.49 &  103.66 \\
$\tilde{\gL}_{te}$ &  109.90 &  \best{107.36} &         106.70 &  106.12 &  \best{105.93} &  106.34 &  105.37 &         104.68 &  \best{104.62} &  105.56 &  \best{103.53} &  104.11 &  104.29 &  104.45 \\
$\text{NLL}_{va}$  &  102.98 &  \best{100.36} &         100.14 &   99.82 &   \best{99.68} &   99.87 &   99.43 &   \best{98.75} &          98.85 &   99.38 &   \best{97.97} &   98.48 &   98.58 &   98.78 \\
$\text{NLL}_{te}$  &  103.28 &  \best{100.79} &         100.48 &  100.16 &  \best{100.00} &  100.18 &   99.90 &          99.23 &   \best{99.21} &   99.80 &   \best{98.48} &   98.80 &   99.04 &   99.22 \\
\bottomrule
\end{tabular}
\end{table*}

\section{Connecting Variance and Lower Bound}
\label{sec:varmin}
Ideally, using a joint proposal allows for modelling the dependency among $\vz_j$'s and thus the negative correlation among the likelihood ratios, but we did not choose to optimize (minimize) variance directly.
Instead we maximize the lower bound. 
We are interested in how the two are connected, i.e. if convergence of one implies another. 
The following is the main takeaway of this section:

\textit{Convergence happening in one mode implies the convergence in the other mode ``if values of $w_j$ and $\log w_j$ cannot be extreme'' (we will assume these two are bounded in some sense). 
However, this is in general not true in the case of importance sampling, where the likelihood ratio is sensitive to the choice of proposal distribution. 
This suggests the density $q$ needs to be controlled in some way (e.g. smoothing the standard deviation of a Gaussian to avoid over-confidence).}

We work in a probability space $(\Omega, \gF, \sP)$. We write $w\in\gL^p$ for $p\geq1$ if $w$ is a random variable and $\E[|w|^p]<\infty$. 
A random sequence $\{w_n\}$ converges in $\gL^p$ to $w$ if $\E[|w_n-w|^p]\rightarrow0$ as $n\rightarrow\infty$.
Convergence in probability and in $\gL^p$ are denoted by $\overset{P}{\rightarrow}$ and $\overset{\gL^p}{\rightarrow}$, respectively. 

Let $\{w_n\}$ be a sequence of random variables, such that $\E[w_n]=p(\vx)$ and thus $\E[\log w_n]\leq \log p(\vx)$.
$w_n$ can be thought of as the likelihood ratio $p(\vx,\vz)/q(\vz)$ where $\vz\sim q(\vz)$, and its expectation wrt $\vz$ is exactly $p(\vx)$ \footnote{The index $n$ can be thought of as an indicator of a sequence of parameterizations of the joint $Q^0$. }.
We'd like to know if convergence of the lower bound $\E[\log w_n]$ to $\log p(\vx)$ (e.g. via maximization of H-IWLB) implies vanishing variance under any assumption, which justifies the use of a joint/hierarchical proposal to reduce variance at a potentially faster rate.
However, it is in general untrue that smaller $\Var(w_n)$, smaller $\Var(\log w_n)$ and larger $\E[\log w_n]$ can imply one another \footnote{\citet{klys2018joint} also attempt to derive a bound, but their result does not hold without making any assumption. \citet{maddison2017filtering} also require uniform integrability to establish consistency.}. 
Instead, we discuss their limiting behavior, and conclude that the condition $\Var(\log w_n)$ converges to zero sits somewhere between $\gL^1$-convergence and $\gL^2$-convergence of $\log w_n$. 
To do so, we require a bit more control over the sequences $w_n$ and $\log w_n$, such as boundedness. 
The first implication of the conclusion is that if the variance of $\log w_n$ vanishes and the sequences are bounded in some sense, $\E[\log w_n]$ also converges to $\log p(\vx)$ (see below). 
The second implication is that if $\log w_n$ converges to $\log p(\vx)$ ``more uniformly'' ($\gL^2$-convergence), then the variance of $\log w_n$ also converges to zero. 
We will discuss why boundedness control is required at the end of this section from the f-divergence perspective.

Now we turn to the analysis on the variance of $\log w_n$.
Doing so allows us to interpret $\log w_n$ as an estimator for $\log p(\vx)$, and we would like to answer the following question: if the variance of $\log w_n$ converges to zero, does $\E[\log w_n]$ converge to $\log p(\vx)$? 
The answer is positive given some boundedness condition, and the result suggests that if the variance of the estimator is infinitesimally small, so is the bias. 
We now state the result:
\begin{restatable}{prop}{varprob}
\label{prop:varprob}
Assume $w_n\in \gL^{1}$ with $w_n>0$, $\log w_n \in \gL^2$ and $\E[w_n]=c$ for some $c>0$, for all $n\geq1$. 
If $\{w_n\}$ is uniformly integrable and $\{\log w_n\}$ is bounded in $\gL^1$,
$$\lim_{n\rightarrow\infty}\textnormal{Var}(\log w_n)=0
\,\Rightarrow\,
\log w_n \overset{P}{\rightarrow} \log c$$
\end{restatable}
\vspace*{-0.3cm}
The proposition tells us that if we look at $\log w_n$ as an estimator for some constant $\log c$ (e.g. $\log p(\vx)$), if $\E[w_n]=c$, then the vanishing variance of $\log w_n$ implies consistency.

With the same integrability condition on $\log w_n$, we can conclude $\E[\log w_n]$ converges to $\log p(\vx)$.
\begin{corollary}
\label{cor:l1}
With the same condition as Proposition~\ref{prop:varprob}, if we further assume $\{\log w_n\}$ is uniformly integrable,
$$\lim_{n\rightarrow\infty}\textnormal{Var}(\log w_n)=0
\,\Rightarrow\,
\log w_n \overset{\gL^1}{\rightarrow} \log c$$
In particular, $\E[\log w_n]\rightarrow \log p(\vx)$ as $n\rightarrow\infty$.
\end{corollary}

Finally, the following result shows that $\gL^2$-convergence is sufficient for the convergence of variance of $\log w_n$ to zero. 
\begin{restatable}{prop}{qm}
\label{prop:qm}
With the same condition as Corollary~\ref{cor:l1}:
$$\log w_n \overset{\gL^2}{\rightarrow} \log c
\,\Rightarrow\,
\lim_{n\rightarrow\infty}\textnormal{Var}(\log w_n)=0
$$
\end{restatable}
Now, we turn back to the case of variational inference where $w=p/q$ (assume $p$ is normalized for the ease of exposition) and discuss the difficulty in analyzing the relationship between variance of $w$ and the lower bound $\E[\log w]$.
It is because variance is more sensitive to large likelihood ratio, whereas lower density region under $q$ does not take a heavy toll on the lower bound.
More concretely, since $\E_q[w]=1$, 
$\Var(w)
=\E_q[({p}/{q})^2]-1
=D_{\mathrm{\chi}^2}(p||q)$, 
%
where $D_{\mathrm{\chi}^2}$ is the $\mathrm{\chi}^2$-divergence.
Our insight is that the $\mathrm{\chi}^2$-divergence is an upper bound on the forward KL divergence $D_{\mathrm{KL}}(p||q)$ (see appendix for a proof).
This means when variance of $w$ decreases, the $\mathrm{\chi}^2$-divergence and the forward KL-divergence also decrease. 
However, decrease in the forward KL does not impliy decrease in the reverse KL, $D_{\mathrm{KL}}(q||p)$, which is what maximization of the lower bound $\E_q[\log w]$ is equivalent to. 
This can be explained by the characteristic function $f$ that is used in the f-divergence family:
$D_{f}(p||q):=\int f\left(\frac{p(\vz)}{q(\vz)}\right)q(\vz) d\vz$, where $f$ is a convex function such that $f(1)=0$. 
For the forward KL and reverse KL, the corresponding convex functions are $f_F(w)=w\log w$ and $f_R(w)=-\log w$, respectively. 
When $w\rightarrow0$, $f_F(w)\rightarrow0$ and $f_R(w)\rightarrow\infty$. 
When $w\rightarrow\infty$, $f_F(w)\rightarrow\infty$ and $f_R(w)\rightarrow-\infty$.
This means the forward KL will not reflect it when $p\ll q$, but will be high when $p\gg q$. 
The reverse KL on the other hand will explode when $p\gg q$ and can tolerate $p\ll q$.
These are known as the inclusive and exclusive properties of the two KLs \citep{minka2005divergence}. 
$\mathrm{\chi}^2$-divergence is defined by the convex function $f_\chi(w)=w^2-1$, which asymptotically bounds $f_F$: $\lim_{w\rightarrow\infty}{f_\chi(w)}/{f_F(w)}=\infty$. 
This means it will incur an even higher cost than the forward KL if $p\gg q$. 
See the Figure~\ref{fig:fdiv} in the appendix for an illustration. 
The boundedness assumption on $\log w_n$ and $w_n$ made in this section makes sure it is less likely that $w_n$ will be arbitrarily large or close to zero.



\begin{table*}[h]
	\centering
    \caption{Variational Autoencoders trained on the Caltech101 Silhouettes dataset~\citep{marlin2010inductive}.
    Each hyperparameter is run 3 times to carry out the statistics ($\mu$:$\sigma$).
    (*) indicates the encoder is updated twice on the same minibatch before the decoder is updated once. }
    \label{tb:caltech}
    \scriptsize
\begin{tabular}{rcccccccc}
\toprule
{} &   model & $\alpha$ & $K$ & $\tilde{\gL}_{tr}$ &    $\tilde{\gL}_{va}$ & $\tilde{\gL}_{te}$ &         $\text{NLL}_{va}$  &         $\text{NLL}_{te}$  \\
\midrule
 &   IWAE &    - &   1 &  123.74:1.56 &  128.87:1.47 &  129.71:1.47 &  117.88:1.50 &  118.61:1.44 \\
 &   IWAE &    - &   2 &  118.67:2.16 &  124.78:2.20 &  125.56:2.12 &  114.24:2.10 &  114.79:2.11 \\
 &   IWAE &    - &   5 &  112.50:1.16 &  120.12:0.43 &  120.90:0.40 &  109.87:0.68 &  110.49:0.64 \\
 &  IWAE (h) &    - &   1 &  113.54:3.06 &  121.20:1.13 &  121.81:1.23 &  109.28:1.38 &  109.71:1.35 \\
 &  IWAE (h) &    - &   2 &  111.94:2.46 &  118.59:1.24 &  119.51:1.13 &  107.52:1.36 &  108.18:1.38 \\
 &  IWAE (h) &    - &   5 &  108.20:0.58 &  115.58:1.08 &  116.40:1.11 &  105.16:0.77 &  105.70:0.78 \\
\midrule
  &  H-IWAE &    0 &   2 &  108.69:1.14 &  118.35:0.69 &  119.18:1.01 &  106.93:0.59 &  107.51:0.74 \\
  &  H-IWAE &    0 &   5 &  108.97:1.15 &  116.84:0.92 &  117.55:0.79 &  106.41:0.88 &  106.83:0.63 \\

 &  H-IWAE &    1 &   2 &  111.57:0.88 &  118.03:0.74 &  118.78:0.71 &  107.01:0.79 &  107.62:0.71 \\
  &  H-IWAE &    1 &   5 &  108.96:0.32 &  116.27:0.39 &  117.05:0.15 &  106.03:0.21 &  106.62:0.09 \\

  &  H-IWAE &    3 &   2 &  111.02:0.72 &  118.02:0.63 &  119.05:0.45 &  107.21:0.46 &  107.90:0.35 \\
  &  H-IWAE &    3 &   5 &  109.66:1.07 &  116.96:0.66 &  117.80:0.73 &  106.41:0.35 &  107.13:0.41 \\

(*) &  H-IWAE &    1 &   2 &  108.31:1.68 &  115.31:1.08 &  116.07:0.94 &  104.27:0.93 &  104.88:0.87 \\
(*) &  H-IWAE &    1 &   5 &  108.03:1.18 &  113.59:0.70 &  114.07:0.75 &  103.74:0.59 &  104.16:0.64 \\
\bottomrule
\end{tabular}
\end{table*}

\section{Related Work}
\label{sec:rw}
Much work has been done on improving the variational approximation, e.g. by using a more complex form of $q$ \citep{rezende2015variational,kingma2016improved,huang2018neural,ranganath2016hierarchical,maaloe2016auxiliary,miller2016variational} in place of the conventional normal distribution.
Along side the work of~\citet{burda2015importance}, \citet{mnih2016variational} and \citet{bornschein2014reweighted,le2018revisiting} also apply importance sampling to learning discrete latent variables and modifying the \textit{wake-sleep algorithm}, respectively.
\citet{cremer2017reinterpreting,bachman2015training} interpret IWAE as using a corrected proposal, whereas \citet{nowozindebiasing} interprets the IWLB as a biased estimator of the marginal log likelihood and propose to reduce the bias using the \textit{Jackknife method}. 
However, \citet{rainforth2018tighter} realize the \textit{signal-to-noise ratio} of the gradient of the inference model vanishes as $K$ increase, as the magnitude of the expected gradient decays faster than variance. 
\citet{tucker2018doubly} propose to use a doubly reparameterized gradient to mitigate this problem. 

Closest to out work is that of~\citet{klys2018joint}, where they propose to explore multivariate normal distribution as the joint proposal. 
\citet{domke2018importance} integrate defensive sampling into variational inference, and \citet{wu2018differentiable} propose to use a differentiable antithetic sampler. 
\citet{yin2018semi,molchanov2018doubly,sobolevimportance} propose to use multiple samples to better estimate the marginal distribution of a hierarchical proposal.


\section{Experiments}
\label{sec:exp}

We first demonstrate the effect of sharing a common random number on the dependency among the multiple proposals, and then apply the amortized version of hierarchical importance sampling (that is, H-IWAE) to learning a deep latent Gaussian models.
The details of how the inference model $Q^0(\cdot|\vx)$ is parameterized can be found in Appendix~\ref{app:addd}. 

\subsection{Effect of Common Random Number}
\label{sec:exp1}
In this section, we analyze the effect of training with common random number, $\vz_0$. 
As mentioned in Section~\ref{sec:hiwae}, the correlation among $\vz_j$'s is induced by tying $\vz_0$ as a common factor. 
If for each index $j$, we draw $\vz_0$ independently from $q_0$, $\vz_j$ will be rendered independent. 
Doing so allows us to test if inference models trained with a common random number tend to output correlated $\vz_j$'s. 
Let the target in Figure~\ref{fig:hiwae_toy} be the posterior distribution $\tilde{p}(\vz)$ (unnormalized), and let $\bar{w}_j=\tilde{p}(\vz_j)r_j(\vz_0|\vz_j)/q_j(\vz_j|\vz_0)q_0(\vz_0)$ and $\bar{w}=\sum \bar{w}_j/K$.
We consider maximizing H-IWLB with two settings: with and without a common $\vz_0$. 
We repeat the experiment 25 times with different random seeds, record the corresponding H-IWLB, variance of $\log \bar{w}$, variance and standard deviation of $\bar{w}$ with common $\vz_0$, as plotted in Figure~\ref{fig:hiwae_dep_toy}.

Qualitatively, we visualize the correlation matrix of $\bar{w}_j$ in Figure~\ref{fig:hiwae_corr_toy}. 
This shows that $\bar{w}_j$'s are indeed negatively correlated with each other when trained with common $\vz_0$.
When trained with independent $\vz_0$ for each $\vz_j$,  the hierarchical sampler tends to find similar solutions and are prone to incurring positively correlated biases (deviation from mean).


\subsection{Effect of Weighting Heuristics}
\label{sec:exp2}
We also analyze the effect of using different weighting heuristics.
We repeat the experiment in the last subsection and apply the power heuristic with $\alpha=0$ and $\alpha=1$ as the weighting scheme, both with common $\vz_0$.
We find that the weighting function $\pi_j$ has a major effect on the shape of the learned proposals $q_j$ (see Figure~\ref{fig:hiwae_toy}). 
With $\alpha=1$, the proposals behave more like a mixture, and the negative correlation is stronger as the proposals ``specialize'' in different regions. 
We also explore training the weighting function in Appendix~\ref{app:adde}.

\subsection{Variational Autoencoder}
\label{sec:exp3}
Our final experiment was to apply hierarchical proposals to learning variational autoencoders on a set of standard datasets, including binarized MNIST~\citep{larochelle2011}, binarized OMNIGLOT~\citep{lake2015human} and Caltech101 Silhouettes~\citep{marlin2010inductive}, using the same architecture as described in~\citet{huang2018neural}. 

Results on the binarized MNIST and OMNIGLOT are in Table~\ref{tb:mnist} and~\ref{tb:omniglot}, respectively. 
We compare with Gaussian IWAE as a baseline. 
The hyperparameters are fixed as follows: minibatch size 64, learning rate $5\times10^{-5}$, linear annealing schedule with 50,000 iterations for the log density terms except $p(\vx|\vz)$ (i.e. KL between $q(\vz|\vx)$ and $p(\vz)$ for VAE), polyak averaging with exponential averaging coefficient 0.998. 
For the MNIST dataset, with $\alpha=0$, arithmetic averaging does not provide monotonic improvement in terms of negative log-likelihood (NLL) when $K$ increases, whereas with $\alpha=1$ or $2$ the performance is better with larger $K$. 
For the OMNIGLOT dataset, the performance is consistently better with larger $K$. 

Table~\ref{tb:caltech} summarizes the experiment on the Caltech101 dataset. 
We compare with the Gausssian IWAE and IWAE with one single hierarchical proposal, dubbed IWAE (h), 
and perform grid search on a set of hyperparameters (Appendix~\ref{app:addd}), each repeated three times to calculate the mean and standard deviation. 
We report the performance of the models by selecting the hyperparameters that correspond to the lowest averaged NLL on the validation set. 
We see that with larger $K$, H-IWAE has an improved performance, but is outperformed by the IWAE with a single hierarchical proposal. 
We speculate this is due to the increased difficulty in optimizing the inference model with the extra parameters for each $q_j$, resulting in a larger amortization bias in inference~\citep{cremer2018inference}. 
To validate the hypothesis, we repeat the experiment of H-IWAE with the balanced heuristic $\alpha=1$, but update the encoder twice on the same minibatch before updating the decoder once. 
This gives us a significant improvement over the learned model (see the last two rows).

\section{Conclusion}
In order to approximate the posterior distribution in variational inference with a more representative empirical distribution, we propose to use a hierarchical meta-sampler.
We derive a variational lower bound as a training objective.
Theoretically, we provide sufficient condition on boundedness to connect the convergence of the variance of the Monte Carlo estimate of the lower bound with the convergence of the bound itself. 
Empirically, we show that maximizing the lower bound implicitly reduces the variance of the estimate. 
Our analysis shows that learning dependency among the joint samples can induce negative correlation and improve the performance of inference.


\nocite{langley00}

\bibliography{main}
\bibliographystyle{icml2019}


\clearpage
\onecolumn
\appendix

\section{Detailed J-IWLB Derivation}
\label{app:jiwlb}

We provide an alternative to the derivation in Section \ref{subsec:jims}. Note that if we can show that
\begin{align*}
  \Esubarg{\gQ}{\sum_{j = 1}^{K} \pi_j\left(\vz_j\right) \frac{p\left(\vx, \vz_j\right)}{q_{j}\left(\vz_j\right)}}
\end{align*}
is equal to $p\left(\vx\right)$, then an application of Jensen's inequality to the log of this quantity gives the desired bound.

First, since linearity holds for any collection of variables, whether or not they are independent, the summation can be taken outside, and since the resulting expectations involve only one 
$\vz_j$ at a time, the expectations can be taken
with respect to their marginals $q_j$. That is,
\begin{align*}
  \Esubarg{\gQ}{\sum_{j = 1}^{K} \pi_j\left(\vz_j\right)\frac{p\left(\vx, \vz_j\right)}{q_j\left(\vz_j\right)}} &= \sum_{j = 1}^{K} \Esubarg{q_j}{\pi_{j}\left(\vz_j\right) \frac{p\left(\vx, \vz_j\right)}{q_j\left(\vz_j\right)}}.
\end{align*}

Each of these $K$ expectations has a simple form,
\begin{align*}
  \Esubarg{q_j}{\pi_{j}\left(\vz_j\right) \frac{p\left(\vx, \vz_j\right)}{q_j\left(\vz_j\right)}} &= \int \pi_j\left(\vz\right)p\left(\vx, \vz\right) d\vz,
\end{align*}
where it's okay to drop the index $j$ previously appended to the $z$'s, because each integral refers to
all values each $z_j$ could possibly take on (not any individual sample of their
values, see Figure \ref{fig:integral_decomposition}). Since $\sum_{j} \pi_j\left(\vz\right) = 1$ pointwise for all
$z$, we can then swap summation and integration (assuming dominated-convergence-style regularity) and find

\begin{align*}
  \sum_{j = 1}^{K} \int \pi_j\left(\vz\right)p\left(\vx, \vz\right) d\vz &= \int \sum_{j=1}^{K} \pi_{j}\left(\vz\right) p\left(\vx, \vz\right) d\vz \\
  &= \int p\left(\vx, \vz\right) d\vz \\
  &= p\left(\vx\right).
\end{align*}

\begin{figure}[ht]
  \centering
  \includegraphics[width=0.4\paperwidth]{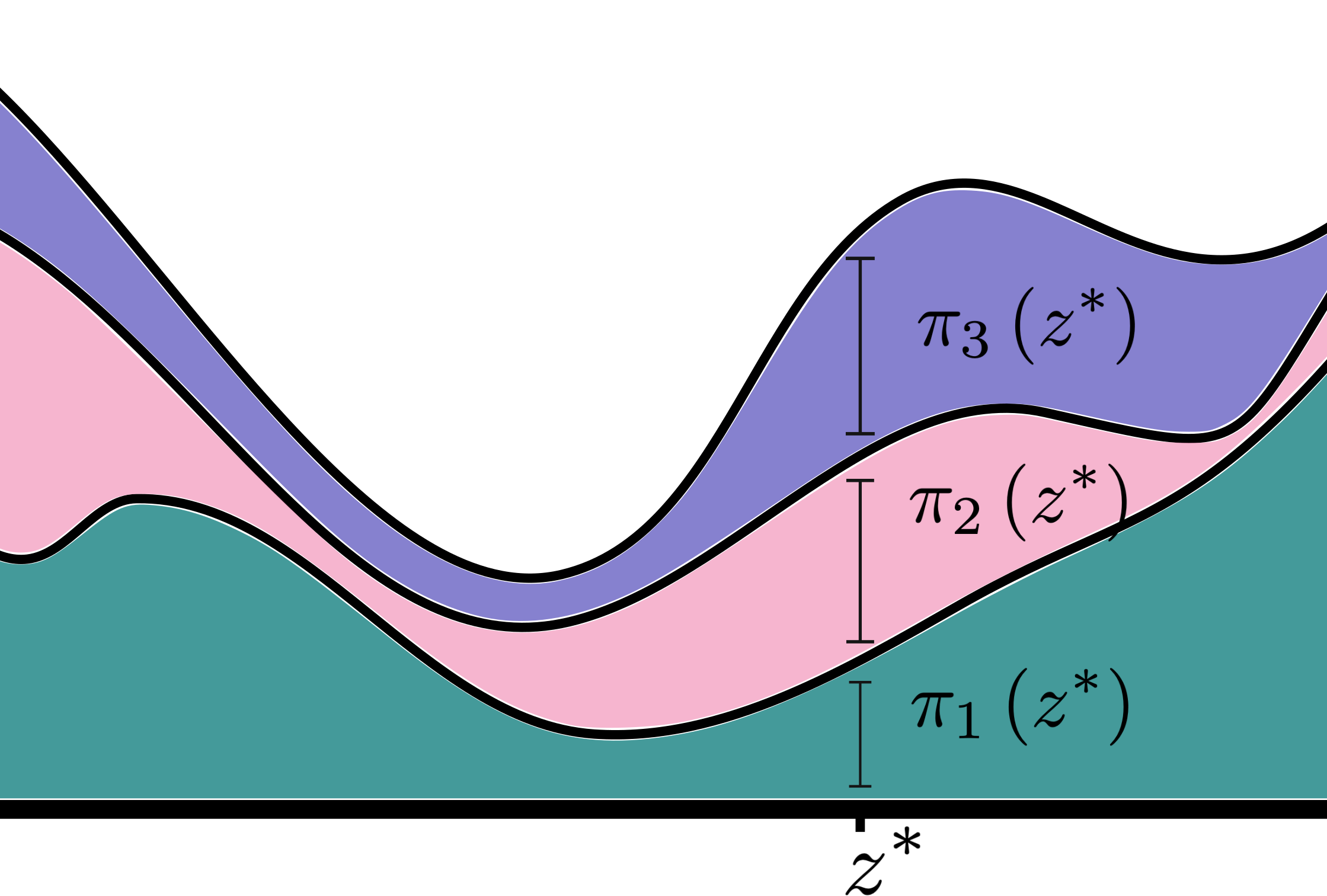}
  \caption{The integral $\int p\left(x, z\right) dz$ can be split into $K$
    terms, where at each $z$ the different terms have a proportion
    $\pi_{j}\left(z\right)$ of the original function
    value. \label{fig:integral_decomposition} }
\end{figure}

\section{Detailed H-IWLB Derivation}
\label{app:hiwlb}
This derivation follows the derivation from the last section, and differs in the last step where we also marginalize out the auxiliary variable $\vz_0$.
\begin{align*}
\gL_K(Q^0) &\leq \log \E_{Q^0} \left[\sum_{j=1}^K\pi_j(\vz_j,\vz_0) \frac{p(\vx,\vz_j)r(\vz_0|\vz_j)}{q_j(\vz_j|\vz_0)q_0(\vz_0)} \right] &\text{(Jensen's Inequality)}\\
&=\log\sum_{j=1}^K \E_{Q^0} \left[ \pi_j(\vz_j,\vz_0) \frac{p(\vx,\vz_j)r(\vz_0|\vz_j)}{q_j(\vz_j|\vz_0)q_0(\vz_0)}\right] &\text{(Linearity of expectation)} \\
&=\log\sum_{j=1}^K \int_{\vz_0,\vz_j} \pi_j(\vz_j,\vz_0) q_j(\vz_j|\vz_0)q_0(\vz_0) \frac{p(\vx,\vz_j)r(\vz_0|\vz_j)}{q_j(\vz_j|\vz_0)q_0(\vz_0)} d\vz_0d\vz_j &\text{(Marginalization of $\vz_{\neg (j\wedge0)}$)} \\
&=\log\int_{\vz_0,\vz} \left(\sum_{j=1}^K\pi_j(\vz,\vz_0)\right) p(\vx,\vz)r(\vz_0|\vz) d\vz_0d\vz &\text{(Identity)} \\
&= \log p(\vx) &\text{(Marginalization of $\pi_j$, $\vz_0$ and $\vz$)} 
\end{align*}

\section{Proofs of Section 4}

\paragraph{Setup.} 
We work in a probability space $(\Omega, \gF, \sP)$. We write $X\in\gL^p$ for $p\geq1$ if $X$ is a random variable and $\E[|X|^p]<\infty$. 
Convergence in probability and in $\gL^p$ are denoted by $\overset{P}{\rightarrow}$ and $\overset{\gL^p}{\rightarrow}$, respectively. 

We start with a classic result. 
Assume $\Var(w_n)\rightarrow0$ as $n\rightarrow\infty$, then by Chebyshev's inequality, $w_n\rightarrow p(\vx)$ in probability, and $\log w_n\rightarrow \log p(\vx)$ in probability by the continuity of $\log$. 
To get $\gL^1$-convergence, i.e. $\E[|\log w_n - \log p(\vx)|]\rightarrow0$, the missing piece that is both necessary and sufficient is the uniform integrability of $\log w_n$, which is a form of boundedness condition (see also Proposition 1 of \citet{maddison2017filtering}). 
It is clear from now that to say something about the expected value we need to bound the random variable in some way.
With the $\gL^1$-convergence, we conclude the expected lower bound (in our case, some form of ELBO) converges to the marginal log-likelihood:
$$|\E[\log w_n]-\log p(\vx)|\leq \E[|\log w_n - \log p(\vx)|]\overset{n\rightarrow\infty}\longrightarrow0$$

\begin{definition}
A family of random variables $\{X_n\}$ is bounded in probability if for any $\epsilon$, there exists $M\geq0$ such that 
$$\sup_{n\geq1}\sP(|X_n|>M)<\epsilon$$
\end{definition}
Note that if $\{X_n\}$ is bounded in $\gL^1$ (i.e. $\sup_{n\geq1}\E[|X_n|]<\infty$), $\{X_n\}$ is also bounded in probability, since by Markov's inequality,
$$\sP(|X_n|>M)<\frac{\E[|X_n|]}{M}\leq\frac{\sup_n\E[|X_n|]}{M}$$
Setting $M=\frac{\sup_n\E[|X_n|]}{\epsilon}$ gives us the uniform bound.

Below, we extend the well known Continuous Mapping Theorem to the difference of random sequences $X_n-Y_n$.
\begin{lemma}
(\textbf{Extended Continuous Mapping Theorem})
Let $\{X_n\}$ and $\{Y_n\}$ be random variables bounded in probability. 
If $f$ is a continuous function, then
$$X_n-Y_n\overset{P}{\rightarrow}0 \Rightarrow f(X_n)-f(Y_n)\overset{P}{\rightarrow}0$$
\end{lemma}
\begin{proof}
Fix $\epsilon>0$. 
Choose $r>0$. 
There exists some positive value $T_{r,\epsilon}$ such that $\sP(|X_n|>T_{r,\epsilon})<\frac{r}{2}\epsilon$ and $\sP(|Y_n|>T_{r,\epsilon})<\frac{r}{2}\epsilon$. 
Within the interval $[-T_{r,\epsilon},T_{r,\epsilon}]$, $g$ is uniformly continuous, so there exists $\delta_{\epsilon, T_{r,\epsilon}}>0$ such that 
$$|x-y|\leq\delta_{\epsilon, T_{r,\epsilon}}\Rightarrow |g(x)-g(y)| \leq \epsilon$$
Now by subadditivity of measure, 
\begin{align*}\sP(|g(X)-g(Y)|>\epsilon)
&\leq \sP(|X_n|>T_{r,\epsilon})+\sP(|Y_n|>T_{r,\epsilon})\\ 
&\qquad+\sP(\{|g(X_n)-g(Y_n)|>\epsilon\}\cap\{|X_n|\leq T_{r,\epsilon}\}\cap\{|Y_n|\leq T_{r,\epsilon}\}) \\
&\leq r\epsilon + \sP(|X_n-Y_n|>\delta_{\epsilon, T_{r,\epsilon}}) 
\end{align*}
The second term goes to $0$ as $n\rightarrow\infty$.
Taking $r$ to $0$ yields the result. 
\end{proof}
Note that the continuity assumption of $f$ can be weakened to assuming the set of discontinuity points of $f$ has measure zero.

Now we restate Proposition~\ref{prop:varprob} below.
\varprob*
\begin{proof}
Let $C_n=\E[\log w_n]$ and $C=\log c$. 
By Chebyshev's inequality, for any $\epsilon>0$, 
$$\lim_{n\rightarrow\infty}\sP(|\log w_n - C_n|>\epsilon)=0$$
which means $\log w_n - C_n\overset{P}{\rightarrow}0$.
Due to $\gL^1$-boundedness, $\log w_n$ is also bounded in probability, and $|C_n|\leq\E[|\log w_n|]\leq\sup_n\E[|\log w_n|]<\infty$.
By the \textit{Extended Continuous Mapping Theorem}, $w_n-\exp(C_n)\overset{P}{\rightarrow}0$. 
Also, $\exp(C_n)=\exp(\E[\log w_n])\leq\exp(|\E[\log w_n]|)$ is bounded, implying $w_n-\exp(C_n)$ is uniformly integrable, so $w_n-\exp(C_n)\overset{\gL^1}{\rightarrow}0$, and
$$\lim_{n\rightarrow\infty}|\E[w_n]-\exp(C_n)|\leq\lim_{n\rightarrow\infty}\E[|w_n-\exp(C_n)|]=0$$
Thus, $\lim_{n\rightarrow\infty}\exp(C_n)=c$ and $\lim_{n\rightarrow\infty} C_n=\log c=C$.

The rest can follow naturally by applying continuous mapping again.
As an alternative, we consider an elementary proof.
Fix $\epsilon>0$ and let $A_n=\{|\log w_n - C|>\epsilon\}$.
Assume $\sP(A_n)\nrightarrow0$.
Then there exists $d'\in(0,1]$ such that $\sP(A_n)\geq d'$ for infinitely many $n$. 
Let $n_k$ be such a subsequence; $\forall\,k\geq1$,
$$\sP(|\log w_{n_k} - C_{n_k}|>\frac{\epsilon}{2}) + \sP(|C_{n_k} - C|>\frac{\epsilon}{2}) \geq \sP(|\log w_{n_k} - C|>\epsilon)\geq d'$$
Since the first term on the LHS converges to zero, 
$\sP(|C_{n_k} - C|>\frac{\epsilon}{2})=1$ asymptotically, as both $C_{n_k}$ and $C$ are constant (event $\{|C_{n_k}-C|>\frac{\epsilon}{2}\}$ is either $\Omega$ or $\emptyset$).
However, $C-C_{n_{k}}\leq \frac{\epsilon}{2}$ for infinitely many $k$. 
Thus, it must be true that $\sP(A_n)\rightarrow0$ as $n\rightarrow\infty$, and $\log w_n\overset{P}{\rightarrow} C$.  
\end{proof}

Proof of Proposition~\ref{prop:qm} is more straightforward. 
We first restate the statement. 
\qm*
\begin{proof}

The variance of $\log w_n$ can be bounded by triangular inequality:
$$\Var(\log w_n)\leq\E[( \log w_n - \log c)^2 ] + (\log c-\E[\log w_n])^2$$
The first terms goes to zero by convergence in $\gL^2$. 
Note that
$|\E[\log w_n] - \log c| \leq \E[|\log w_n - \log c|]\overset{n\rightarrow\infty}\longrightarrow0 $. 
So the second term also converges to zero. 
\end{proof}

\begin{figure*}[!th]
\centering
\includegraphics[width=0.50\textwidth]{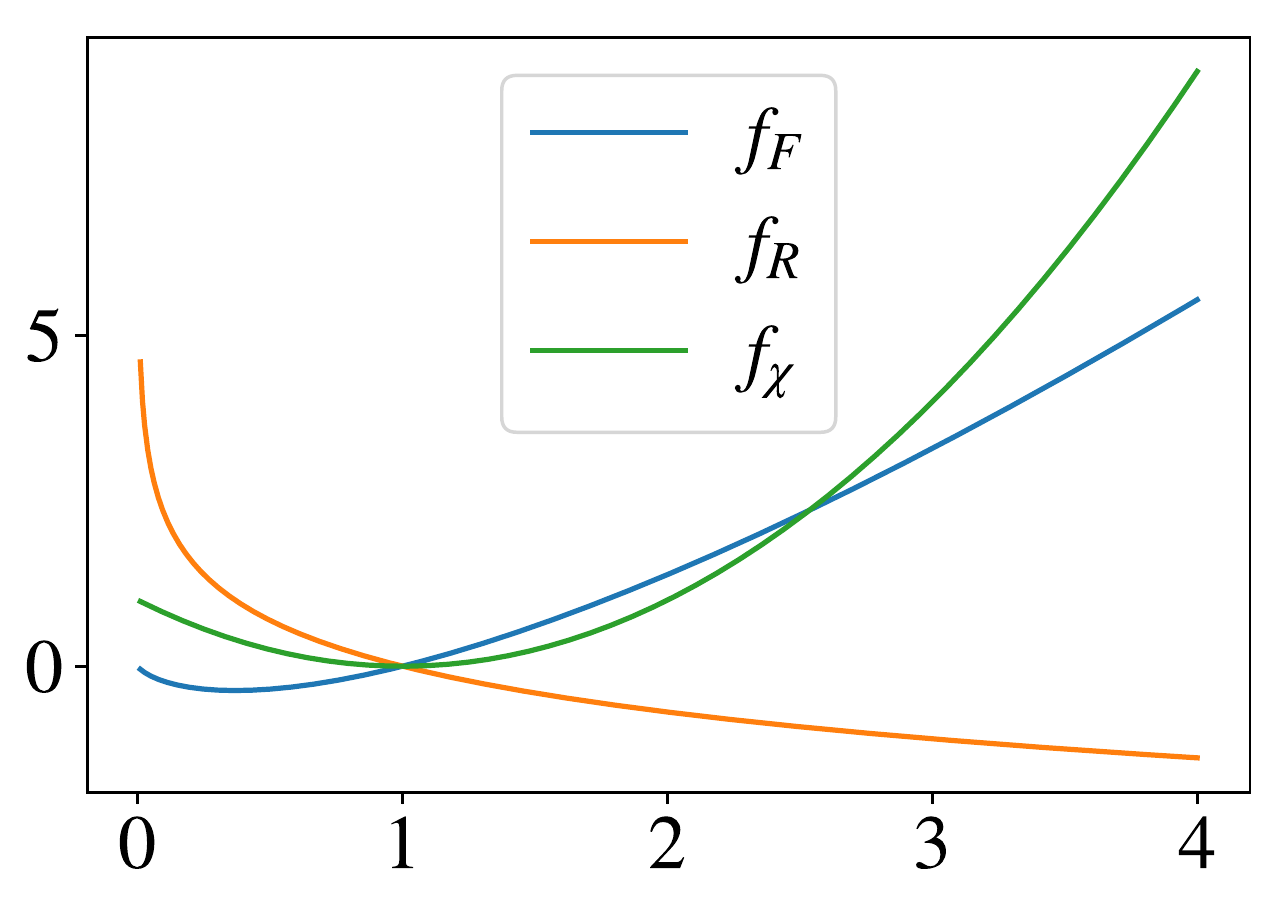}
\caption{The characteristic convex functions for different f-divergences.
$f_F$ and $f_\chi$ are larger when the x-axis value is large, whereas $f_R$ penalizes more when the x-axis value is close to zero. }
\label{fig:fdiv}
\end{figure*}

Finally, we show that $D_{\mathrm{\chi}^2}(p||q)$ is an upper bound on $\KL(p||q)$. 
Generally, let $Q$ and $P$ be probability measures such that $Q\gg P$ ($P$ is absolutely continuous wrt $Q$). 
Since for $x>0$, $\log x \leq x-1$
\begin{align*}
D_{\mathrm{KL}}(P||Q)
&=\int_\Omega \log\left(\frac{d P}{d Q}\right) dP  \\
&\leq \int_\Omega \left(\frac{d P}{d Q}-1\right) dP \\
&= \int_\Omega \left(\frac{d P}{d Q}\right)^2dQ -1 = D_{\chi^2}(P||Q)
\end{align*}

The different characteristic convex functions of the f-divergences are plotted in Figure~\ref{fig:fdiv} for reference.

\section{Additional Experimental Details}
\label{app:addd}
\paragraph{Parameterization of the hierarchical proposals} In all of our experiments, the conditionals $q_j(\vz_j|\vz_0)$ are all normal densities parameterized by a multilayer perceptron (MLP); i.e. $\gN(\vz_j; \mu(\vz_0), \sigma^2(\vz_0))$, where
$$\mu(\vz_0)= (\mW_\mu \vh + \vb_\mu ) + \mW_s \vz_0
\qquad\qquad \sigma(\vz_0)= (\mW_\sigma \vh + \vb_\sigma )^+ 
\qquad\qquad \vh=g(\mW_h\vz_0+\vb_h)$$
where $(\cdot)^+$ denotes the softplus nonlinearity and $g(\cdot)$ is the ELU activation \citep{clevert2015fast}.
We share the hidden units $\vh$ ($\vh\in\sR^{1920}$ for the MNIST and OMNIGLOT experiments) for different $q_j(\vz_j|\vz_0)$'s. 

The conditional $r(\vz_0|\vz_j)$ is also a normal density, with a similar parameterization. 
But when arithmetic averaging ($\alpha=0$) is used, we learn the embedding of each $j$ and conditionally scale and shift the weight norm parameters of the hidden units. 
We also apply the conditional weight norm to condition on the input data $\vx$ in the amortized inference set-up. 

\paragraph{Hyperparameter search of the Caltech101 experiment} We perform grid search on the power set of the following hyperparameters for the Caltech101 experiments:
\begin{enumerate}
    \item learning rate: $[0.0003,0.0001,0.00005,0.0003,0.00001]$
    \item free bits~\citep{kingma2016improved}: $[0.00,0.01]$
    \item polyak averaging: $[0.95,0.99]$
    \item dimensionality of $\vh$: $[500,1024]$ 
\end{enumerate}

\newpage

\section{Additional Experimental Results}
\label{app:adde}
Aside from the power heuristics in Section~\ref{sec:hiwae}, we also explore the possibility to parameterize the weighting function $\pi_j(\vz)$, using an MLP with a softmax output. 
We visualize the learned $\pi_j$ in Figure~\ref{fig:hiwae_toy_learned_pi}.

\begin{figure*}[!th]
\centering
\subfigure{
\includegraphics[width=0.95\textwidth]{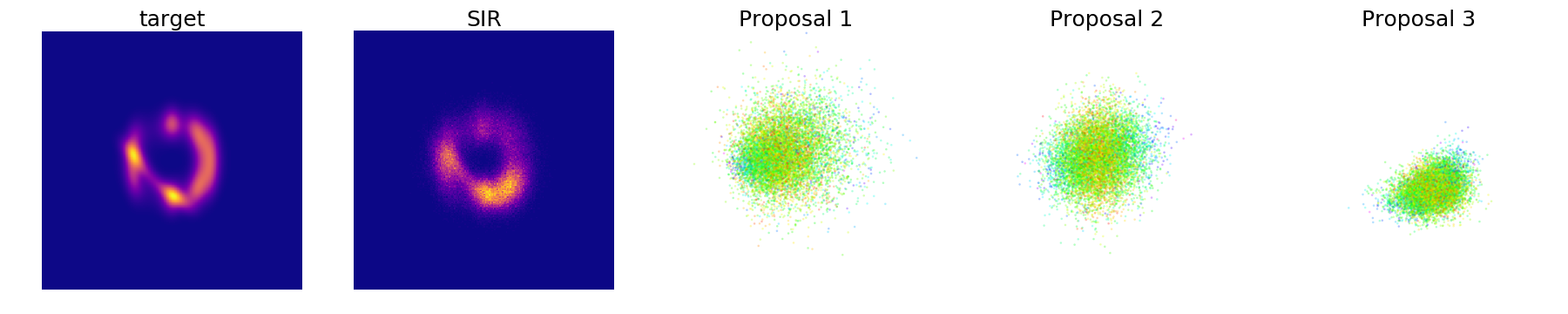}
\label{fig:mulprop_pi_hvi_sub}
}
\subfigure{
\includegraphics[width=0.95\textwidth]{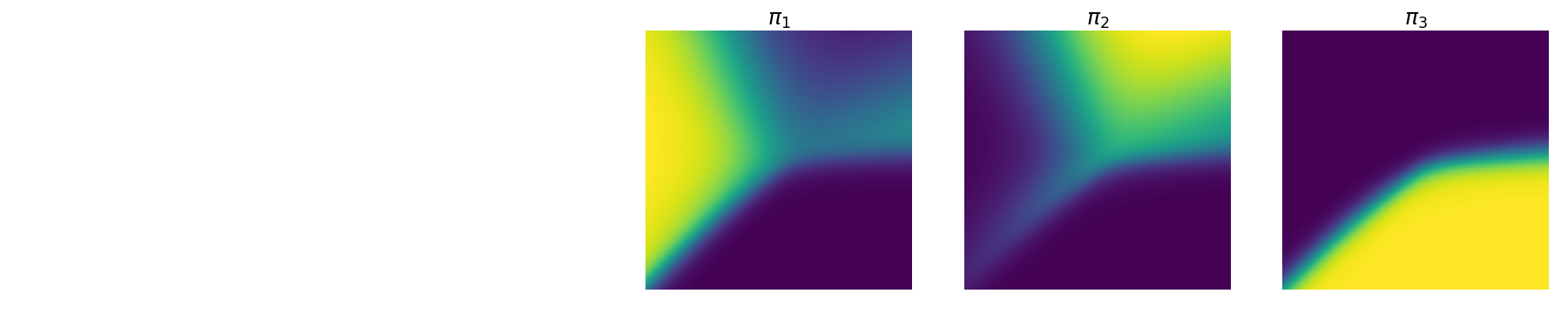}
\label{fig:mulprop_pi_hvi_pi_sub}
}
\subfigure{
\includegraphics[width=0.95\textwidth]{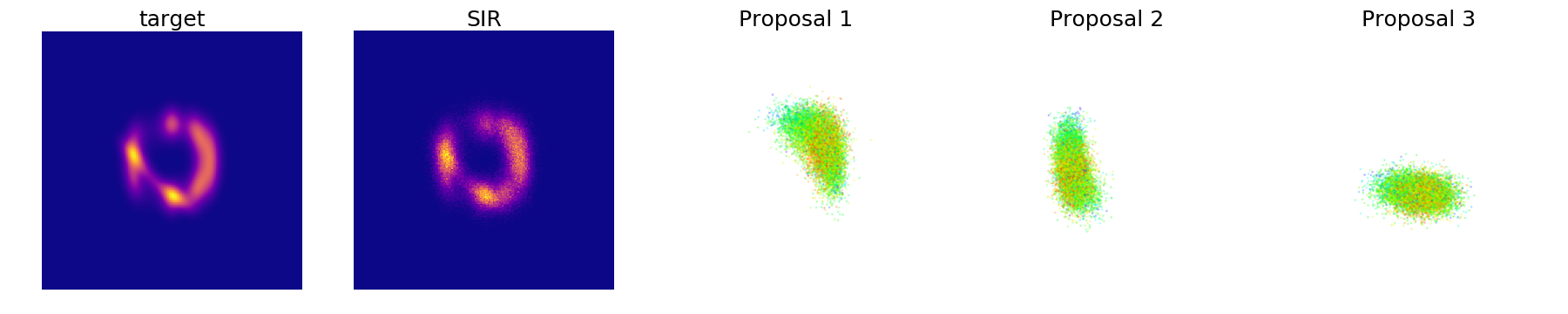}
\label{fig:mulprop_pi_hvi}
}
\subfigure{
\includegraphics[width=0.95\textwidth]{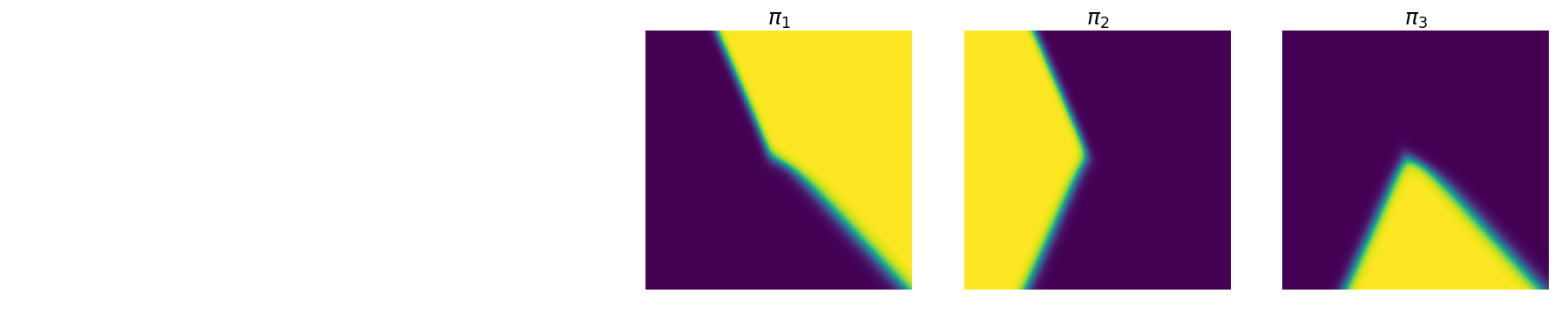}
\label{fig:mulprop_pi_hvi_pi}
}

\caption{Learned hierarchical importance sampling proposals with a learned weighting function $\pi_j(\vz)$ parameterized by a one-hidden-layer MLP. 
The one on top is only trained for 50 iterations with SGD. 
The one below is trained till convergence (5,000 iterations). }
\label{fig:hiwae_toy_learned_pi}
\end{figure*}

\clearpage

\begin{table*}[!h]
\centering
\caption{$K$=1 on MNIST.}
\label{tb:mnist-1}

\begin{tabular}{rcccc}
\toprule
mode               &              IWAE &            H-IWAE &                   H-IWAE &            H-IWAE \\
$\alpha$           &                 - &                 0 &                        1 &                 3 \\
$\tilde{\gL}_{tr}$ &  83.26 \err{0.10} &  82.92 \err{0.17} &  \best{82.63} \err{0.07} &  82.67 \err{0.15} \\
$\tilde{\gL}_{va}$ &  86.57 \err{0.11} &  85.75 \err{0.08} &  \best{85.68} \err{0.06} &  85.72 \err{0.06} \\
$\tilde{\gL}_{te}$ &  86.36 \err{0.15} &  85.50 \err{0.08} &  \best{85.43} \err{0.05} &  85.49 \err{0.07} \\
$\text{NLL}_{va}$  &  82.64 \err{0.11} &  82.24 \err{0.05} &  \best{82.16} \err{0.03} &  82.20 \err{0.04} \\
$\text{NLL}_{te}$  &  82.37 \err{0.12} &  81.96 \err{0.04} &  \best{81.89} \err{0.04} &  81.93 \err{0.03} \\
\bottomrule
\end{tabular}

\end{table*}

\begin{table*}[h]
\centering
\caption{$K$=2 on MNIST.}
\label{tb:mnist-2}

\begin{tabular}{rcccc}
\toprule
mode               &              IWAE &            H-IWAE &                   H-IWAE &            H-IWAE \\
$\alpha$           &                 - &                 0 &                        1 &                 3 \\
$\tilde{\gL}_{tr}$ &  82.36 \err{0.20} &  82.19 \err{0.55} &  \best{82.15} \err{0.60} &  82.43 \err{0.59} \\
$\tilde{\gL}_{va}$ &  85.40 \err{0.05} &  85.05 \err{0.64} &  \best{85.03} \err{0.65} &  85.18 \err{0.55} \\
$\tilde{\gL}_{te}$ &  85.16 \err{0.03} &  84.79 \err{0.63} &  \best{84.76} \err{0.64} &  84.90 \err{0.56} \\
$\text{NLL}_{va}$  &  82.03 \err{0.04} &  81.93 \err{0.42} &  \best{81.88} \err{0.35} &  81.96 \err{0.27} \\
$\text{NLL}_{te}$  &  81.77 \err{0.04} &  81.65 \err{0.42} &  \best{81.60} \err{0.35} &  81.66 \err{0.27} \\
\bottomrule
\end{tabular}

\end{table*}

\begin{table*}[h]
\centering
\caption{$K$=5 on MNIST.}
\label{tb:mnist-5}

\begin{tabular}{rcccc}
\toprule
mode               &              IWAE &            H-IWAE &                   H-IWAE &                   H-IWAE \\
$\alpha$           &                 - &                 0 &                        1 &                        3 \\
$\tilde{\gL}_{tr}$ &  81.48 \err{0.17} &  82.25 \err{0.23} &  \best{81.28} \err{0.14} &         81.32 \err{0.18} \\
$\tilde{\gL}_{va}$ &  84.45 \err{0.06} &  85.05 \err{0.39} &  \best{84.04} \err{0.16} &         84.12 \err{0.16} \\
$\tilde{\gL}_{te}$ &  84.25 \err{0.08} &  84.85 \err{0.41} &  \best{83.79} \err{0.14} &         83.90 \err{0.12} \\
$\text{NLL}_{va}$  &  81.63 \err{0.04} &  82.19 \err{0.27} &         81.42 \err{0.07} &  \best{81.39} \err{0.09} \\
$\text{NLL}_{te}$  &  81.37 \err{0.04} &  81.93 \err{0.28} &         81.16 \err{0.05} &  \best{81.13} \err{0.09} \\
\bottomrule
\end{tabular}

\end{table*}

\begin{table*}[!h]
\centering
\caption{$K$=10 on MNIST.}
\label{tb:mnist-10}

\begin{tabular}{rcccc}
\toprule
mode               &                     IWAE &            H-IWAE &                   H-IWAE &            H-IWAE \\
$\alpha$           &                        - &                 0 &                        1 &                 3 \\
$\tilde{\gL}_{tr}$ &  \best{80.85} \err{0.16} &  82.30 \err{0.93} &         80.89 \err{0.13} &  81.17 \err{0.24} \\
$\tilde{\gL}_{va}$ &         83.84 \err{0.06} &  85.11 \err{0.87} &  \best{83.77} \err{0.23} &  83.95 \err{0.32} \\
$\tilde{\gL}_{te}$ &         83.62 \err{0.03} &  84.86 \err{0.86} &  \best{83.56} \err{0.22} &  83.72 \err{0.34} \\
$\text{NLL}_{va}$  &         81.37 \err{0.05} &  82.42 \err{0.62} &  \best{81.28} \err{0.08} &  81.33 \err{0.13} \\
$\text{NLL}_{te}$  &         81.13 \err{0.02} &  82.17 \err{0.61} &  \best{81.04} \err{0.09} &  81.08 \err{0.15} \\
\bottomrule
\end{tabular}

\end{table*}

\clearpage

\begin{table*}[!tbh]
\centering
\caption{$K$=1 on OMNIGLOT.}
\label{tb:omniglot-1}

\begin{tabular}{rcccc}
\toprule
mode               &               IWAE &             H-IWAE &             H-IWAE &                    H-IWAE \\
$\alpha$           &                  - &                  0 &                  1 &                         3 \\
$\tilde{\gL}_{tr}$ &  106.40 \err{1.75} &  104.57 \err{0.98} &  104.27 \err{0.70} &  \best{103.77} \err{0.49} \\
$\tilde{\gL}_{va}$ &  109.05 \err{1.28} &  106.48 \err{0.48} &  106.26 \err{0.25} &  \best{106.13} \err{0.39} \\
$\tilde{\gL}_{te}$ &  109.90 \err{1.21} &  107.36 \err{0.52} &  107.18 \err{0.13} &  \best{106.97} \err{0.36} \\
$\text{NLL}_{va}$  &  102.98 \err{1.56} &  100.36 \err{0.50} &  100.13 \err{0.39} &   \best{99.90} \err{0.36} \\
$\text{NLL}_{te}$  &  103.28 \err{1.53} &  100.79 \err{0.52} &  100.57 \err{0.24} &  \best{100.50} \err{0.31} \\
\bottomrule
\end{tabular}

\end{table*}

\begin{table*}[h]
\centering
\caption{$K$=2 on OMNIGLOTt.}
\label{tb:omniglot-2}

\begin{tabular}{rcccc}
\toprule
mode               &                      IWAE &             H-IWAE &                    H-IWAE &             H-IWAE \\
$\alpha$           &                         - &                  0 &                         1 &                  3 \\
$\tilde{\gL}_{tr}$ &  \best{102.66} \err{0.59} &  103.24 \err{1.06} &         102.88 \err{0.89} &  103.17 \err{1.22} \\
$\tilde{\gL}_{va}$ &         105.85 \err{0.19} &  105.27 \err{0.53} &  \best{105.18} \err{0.69} &  105.56 \err{0.88} \\
$\tilde{\gL}_{te}$ &         106.70 \err{0.12} &  106.12 \err{0.39} &  \best{105.93} \err{0.73} &  106.34 \err{0.93} \\
$\text{NLL}_{va}$  &         100.14 \err{0.24} &   99.82 \err{0.60} &   \best{99.68} \err{0.40} &   99.87 \err{0.63} \\
$\text{NLL}_{te}$  &         100.48 \err{0.34} &  100.16 \err{0.50} &  \best{100.00} \err{0.49} &  100.18 \err{0.66} \\
\bottomrule
\end{tabular}

\end{table*}

\begin{table*}[h]
\centering
\caption{$K$=5 on OMNIGLOT.}
\label{tb:omniglot-5}

\begin{tabular}{rcccc}
\toprule
mode               &               IWAE &                    H-IWAE &                    H-IWAE &             H-IWAE \\
$\alpha$           &                  - &                         0 &                         1 &                  3 \\
$\tilde{\gL}_{tr}$ &  102.01 \err{1.05} &  \best{101.35} \err{1.48} &         101.37 \err{1.72} &  102.78 \err{0.87} \\
$\tilde{\gL}_{va}$ &  104.48 \err{0.58} &         103.76 \err{0.73} &  \best{103.72} \err{0.94} &  104.70 \err{0.46} \\
$\tilde{\gL}_{te}$ &  105.37 \err{0.52} &         104.68 \err{0.67} &  \best{104.62} \err{0.80} &  105.56 \err{0.49} \\
$\text{NLL}_{va}$  &   99.43 \err{0.64} &   \best{98.75} \err{0.79} &          98.85 \err{0.92} &   99.38 \err{0.36} \\
$\text{NLL}_{te}$  &   99.90 \err{0.67} &          99.23 \err{0.76} &   \best{99.21} \err{0.95} &   99.80 \err{0.35} \\
\bottomrule
\end{tabular}

\end{table*}
\begin{table*}[!h]
\centering
\caption{$K$=10 on OMNIGLOT.}
\label{tb:omniglot-10}

\begin{tabular}{rcccc}
\toprule
mode               &                      IWAE &             H-IWAE &             H-IWAE &             H-IWAE \\
$\alpha$           &                         - &                  0 &                  1 &                  3 \\
$\tilde{\gL}_{tr}$ &   \best{99.71} \err{0.95} &   99.75 \err{1.31} &  100.81 \err{1.24} &  101.29 \err{0.51} \\
$\tilde{\gL}_{va}$ &  \best{102.67} \err{0.39} &  103.21 \err{0.77} &  103.49 \err{0.63} &  103.66 \err{0.18} \\
$\tilde{\gL}_{te}$ &  \best{103.53} \err{0.23} &  104.11 \err{0.85} &  104.29 \err{0.65} &  104.45 \err{0.17} \\
$\text{NLL}_{va}$  &   \best{97.97} \err{0.47} &   98.48 \err{0.67} &   98.58 \err{0.50} &   98.78 \err{0.30} \\
$\text{NLL}_{te}$  &   \best{98.48} \err{0.34} &   98.80 \err{0.64} &   99.04 \err{0.55} &   99.22 \err{0.25} \\
\bottomrule
\end{tabular}

\end{table*}

\end{document}